%% file: neurips_2026.tex
\documentclass{article}

    \PassOptionsToPackage{numbers, compress}{natbib}
\usepackage[preprint]{neurips_2026}

\usepackage[utf8]{inputenc} 
\usepackage[T1]{fontenc}    
\usepackage{hyperref}       
\usepackage{url}            
\usepackage{booktabs}       
\usepackage{amsfonts}       
\usepackage{nicefrac}       
\usepackage{microtype}      
\usepackage{xcolor}         

\usepackage{colortbl}
\usepackage[table]{xcolor}
\usepackage{array}
\usepackage{pifont}
\usepackage{xspace} 

\makeatletter
\DeclareRobustCommand\onedot{\futurelet\@let@token\@onedot}
\def\@onedot{\ifx\@let@token.\else.\null\fi\xspace}

\def\eg{\emph{e.g}\onedot}

\makeatother

\usepackage{amsmath}
\usepackage{amssymb}
\usepackage{mathtools}
\usepackage{amsthm}

\usepackage{enumitem}
\usepackage{multirow}
\usepackage{wrapfig}

\usepackage{graphicx}
\usepackage{subcaption} 

\usepackage[capitalize,noabbrev]{cleveref}

\theoremstyle{plain}
\newtheorem{theorem}{Theorem}[section]

\newtheorem{lemma}[theorem]{Lemma}

\theoremstyle{definition}

\newtheorem{assumption}[theorem]{Assumption}
\theoremstyle{remark}
\newtheorem{remark}[theorem]{Remark}

\newcommand{\PreserveBackslash}[1]{\let\temp=\\#1\let\\=\temp}
\newcolumntype{C}[1]{>{\PreserveBackslash\centering}p{#1}}
\newcolumntype{R}[1]{>{\PreserveBackslash\raggedleft}p{#1}}
\newcolumntype{L}[1]{>{\PreserveBackslash\raggedright}p{#1}}

\title{Alignment-Guided Flow Transformer for Efficient \\ Vision–Language–Action Policy Learning}

\author{%
  \bfseries Shengchao Hu\textsuperscript{1} \quad
  Peng Wang\textsuperscript{2} \quad
  Qiyang Zhou\textsuperscript{2} \quad
  Guodong Zheng\textsuperscript{2} \quad
  Yuqi Huang\textsuperscript{1} \\
  \bfseries Li Shen\textsuperscript{2} \quad
  Ya Zhang\textsuperscript{1} \quad
  Dacheng Tao\textsuperscript{3} \\[2mm]
  \normalfont
  $^{1}$Shanghai Jiao Tong University \quad
  $^{2}$Shenzhen Campus of Sun Yat-sen University \\
  $^{3}$Nanyang Technological University \\[2mm]
  \ttfamily
  charles-hu@sjtu.edu.cn \quad mathshenli@gmail.com
}

\begin{document}

\maketitle

\input{sections/0_abstract}
\input{sections/1_intro}
\input{sections/2_ReW}
\input{sections/3_Method}
\input{sections/4_Exp}
\input{sections/5_Con}

\bibliography{ref}
\bibliographystyle{plainnat}


\newpage
\appendix

\input{sections/6_suppl}


\newpage
\input{checklist.tex}

\end{document}

%% file: sections/0_abstract.tex
\begin{abstract}
Recent advances in Vision–Language–Action (VLA) models point toward general-purpose robotic intelligence by unifying perception, instruction, and control.
Despite impressive progress, existing VLA models often adapt poorly due to \emph{tri-modal misalignment} among vision, language, and action, which weakens action grounding and hurts generalization and fine-tuning efficiency.
In this work, we present Alignment-Guided Flow Transformer (AGFT), a novel framework that explicitly enforces tri-modal alignment through a dedicated alignment loss, bridging the representational gap across modalities and enhancing task adaptation.
While prior research has predominantly emphasized bi-modal vision--language alignment, we systematically formalize and study tri-modal alignment in VLA models, and provide both ablations and analysis to isolate its role in improving adaptation and robustness.
To further accelerate deployment, we adopt a flow-matching objective, enabling substantially fewer inference steps than diffusion-based policies while maintaining accuracy.
Theoretically, we establish a quantitative connection between the tri-modal alignment gap and the optimization tightness of flow matching; empirically, experiments on the extensive benchmark show that AGFT achieves superior success rates and lower inference latency compared to SOTA baselines, underscoring tri-modal alignment as a key ingredient for scaling robust VLA manipulation.
\end{abstract}

%% file: sections/1_intro.tex
\section{Introduction}
\label{sec:intro}

The recent emergence of Vision–Language–Action (VLA) models \citep{brohan2022rt, zitkovich2023rt, kim2024openvla, black2024pi_0} has brought the robotics community a step closer to developing general-purpose manipulation policies capable of executing diverse tasks across varied environments. Inspired by the success of foundation models in large-scale vision–language learning \citep{radford2021learning, touvron2023llama, karamcheti2024prismatic}, VLA models extend these pre-trained multimodal architectures by attaching an action head that grounds perception and language understanding into motor control. This design allows VLAs to inherit the semantic reasoning, scene understanding, and instruction-following capabilities of visual–language models (VLMs), thereby enabling effective imitation learning at scale.

A typical VLA policy is constructed through a two-stage paradigm: (i) a pre-training phase where a Transformer-based architecture is exposed to large cross-embodiment datasets \citep{o2024open, khazatsky2024droid}, learning a general visuomotor prior across heterogeneous robots, and (ii) an adaptation phase wherein the model is fine-tuned on limited task-specific demonstrations to perform precise manipulation behaviors. Despite their promise, recent VLA architectures can remain fragile under task and environment distribution shifts, as the interdependence between visual perception, linguistic intent, and generated actions is often only weakly aligned. Subtle misalignment—such as inconsistent grounding between textual goals and corresponding visual cues—can cause incoherent or unstable behaviors, especially in long-horizon, compositional tasks.

\begin{wrapfigure}{r}{0.5\textwidth}
    \includegraphics[width=0.5\textwidth]
    {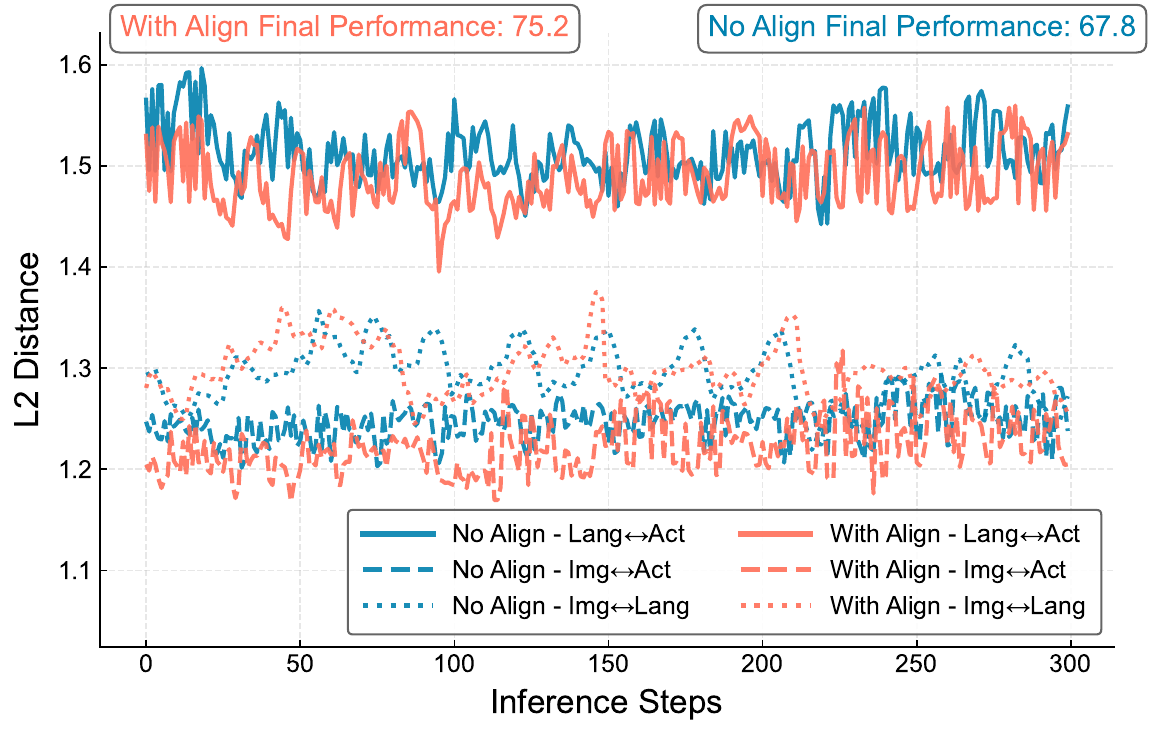} 
    \caption{The evolution of L2 distance between language–action, image–action, and image-language embeddings during the inference process, along with their corresponding task performance on the Libero‑long benchmark. Lower distance indicates a stronger consistency between modalities, which correlates with improved final evaluation performance. 
    }\label{fig:motivation}
\end{wrapfigure}

Existing research has largely focused on improving architectural scalability (e.g., autoregressive or diffusion-based policy heads) \citep{team2024octo, hou2025dita} or data efficiency (e.g., cross-task or cross-robot adaptation) \citep{kim2024openvla, pertsch2025fast}. However, the critical problem of multimodal alignment—that is, ensuring that visual observations, linguistic instructions, and continuous motor actions remain geometrically and semantically consistent within a shared representational space—has been significantly underexplored. Without explicit alignment, policies may learn spurious correlations in multimodal features, leading to degraded generalization, action ambiguity, and sensitivity to environmental changes \citep{zhang2025align}.
As VLA complexity scales with data and parameters, alignment—not just modeling capacity—becomes the true bottleneck for generalization.
As illustrated in Figure~\ref{fig:motivation}, a higher degree of cross-modal alignment exhibits a strong positive correlation with overall policy performance, underscoring the importance of alignment-aware learning objectives in scalable embodied systems.

Beyond alignment, most current policies depend on diffusion-based action heads to generate continuous trajectories \citep{chi2025diffusion, hou2025dita, qu2025spatialvla}. While diffusion models have proven successful in visual generative tasks, their iterative denoising process is computationally expensive and ill-suited to real-time control.
Moreover, the stochastic denoising process in diffusion can make it harder to maintain consistent cross-modal representations during generation, since it relies on repeatedly refining fused embeddings \citep{zhang2025align}.
Consequently, jointly attaining representational integrity and computational efficiency remains an open challenge for large-scale embodied policy learning and warrants deeper investigation.

To address these limitations, we propose the Alignment-Guided Flow Transformer (AGFT), a flow-based VLA policy that couples deterministic flow matching with explicit \emph{action-centered tri-modal alignment} for efficient and coherent action generation.
AGFT models the transport field between latent noise and target actions via a flow-matching objective, enabling fast and stable inference without the costly iterative sampling of diffusion. 
Simultaneously, AGFT enforces semantic consistency across modalities using contrastive losses among image, language and action, ensuring that linguistic intent and visual context are both faithfully reflected in motor outputs. 
Together, these objectives promote a harmonized vision-language-action representation that remains aligned throughout perception, reasoning, and execution.
We further provide theoretical analysis showing how alignment regularization tightens the optimization bounds of the flow-matching objective.
Empirically, on the LIBERO \citep{liu2023libero} and CALVIN \citep{mees2022calvin} benchmarks, AGFT improves success rates and inference efficiency over diffusion-based DiT fine-tuning baselines while remaining competitive with strong VLA baselines.
Finally, ablation studies isolate the contribution of the contrastive loss and corroborate its effectiveness.


Our contributions are summarized as follows:
\begin{itemize}[leftmargin=10pt]
    \item \textbf{Alignment-Guided Flow Transformer (AGFT):} We propose a novel flow-based vision–language–action architecture that eliminates the need for iterative diffusion sampling through deterministic flow matching, enabling fast and high-fidelity action generation.
    \item \textbf{Action-centered tri-modal alignment for VLA:} We identify \emph{tri-modal misalignment} as a VLA-specific bottleneck and introduce an explicit, contrastive alignment objective that primarily aligns vision-action and language-action representations (with vision-language as an auxiliary constraint), improving robustness under task and environment shifts.
    \item \textbf{Theory linking alignment to flow matching:} We provide formal upper and lower bounds for the flow-matching objective and show how reducing the tri-modal alignment gap tightens these bounds, offering theoretical support for alignment regularization.
    \item \textbf{Comprehensive empirical validation:} Extensive experiments on the LIBERO and CALVIN benchmark demonstrate that AGFT remains competitive with strong VLA baselines, delivering enhanced multimodal coherence and significantly faster inference.
\end{itemize}

\begin{figure*}
    \centering
    \includegraphics[width=1.0\linewidth]{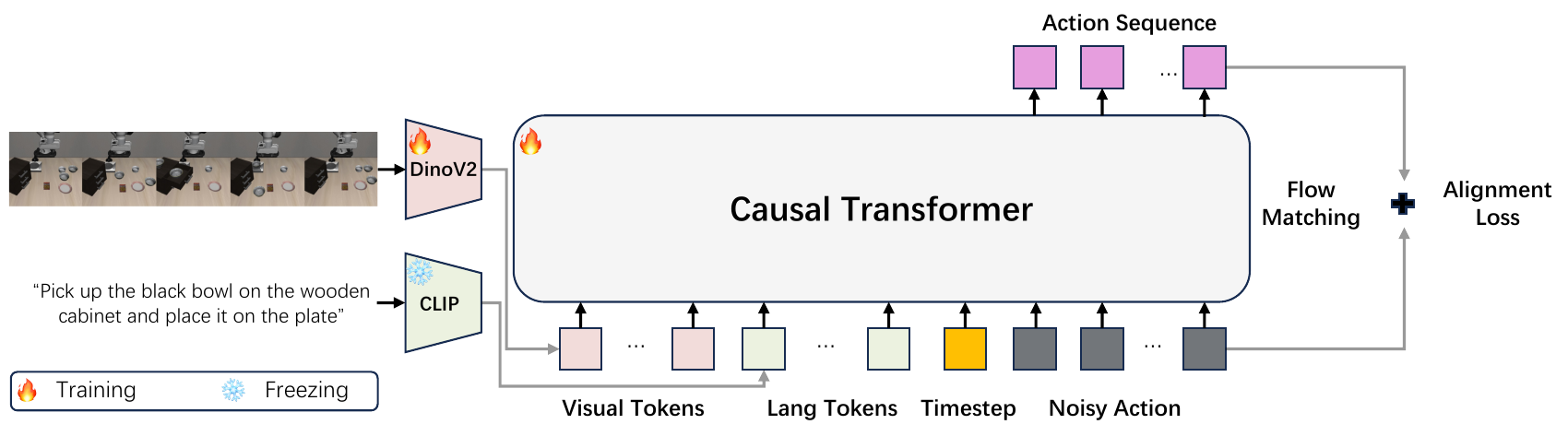}
    \vspace{-.6cm}
    \caption{\textbf{Overview of the proposed AGFT architecture.} The model adopts a Transformer-based policy design that integrates a pretrained CLIP encoder for language understanding and a DINOv2 backbone for visual representation learning. The extracted instruction tokens, image embeddings, timestep encodings, and noised action vectors are concatenated to form a unified token sequence, which is processed by the Transformer to predict clean actions. Training is jointly guided by a flow matching objective for efficient action generation and an alignment loss that enforces semantic consistency across vision, language, and action modalities. } 
    \label{fig:architecture}
    \vspace{-.4cm}
\end{figure*}




%% file: sections/2_ReW.tex
\vspace{-.2cm}
\section{Related Work}
\label{sec:ReW}

\subsection{Vision-Language-Action Models}

VLA models extend pre-trained VLMs by incorporating action heads to generate robot controls, enabling scalable policy learning from large heterogeneous data \citep{black2024pi_0, team2024gemma, chi2025diffusion, bjorck2025gr00t}. 
Early works such as RT-1 and RT-2 adopt a two-stage design in which the VLM first produces latent tokens that are subsequently mapped to discretized actions through an MLP decoder \citep{brohan2022rt, zitkovich2023rt}.
Subsequent approaches scale backbone architectures and vision encoders (\eg, OpenVLA) and explore latent action spaces to improve generalization (\eg, LAPA) \citep{kim2024openvla, ye2024latent, touvron2023llama, oquab2023dinov2, zhai2023sigmoid}. 
To enable high-frequency control and efficient deployment, token compression and action chunking techniques (\eg, $\pi_0$‑FAST, OpenVLA‑OFT) are introduced to reduce inference latency while preserving representational expressiveness \citep{pertsch2025fast, kim2025fine}. 
In parallel, other efforts model actions as continuous trajectories via diffusion or flow-matching objectives, supporting smooth and multimodal control at the expense of iterative inference or increased training complexity \citep{chi2025diffusion, liu2024rdt, hou2025dita}.

Recent research has converged on two key design axes: action representation and its integration with the backbone VLM \citep{ma2024survey, zhong2025survey}. 
Token-based autoregressive policies treat control as discrete sequence modeling \citep{brohan2022rt, chen2021decision}, benefiting from pretrained cross-modal features but facing serial decoding latency and error accumulation \citep{shridhar2023perceiver}. 
In contrast, trajectory-based methods with hierarchical or hybrid architectures (\eg $\pi_0$, DexVLA, HybridVLA) capture long-horizon dynamics more efficiently \citep{black2024pi_0, wen2025dexvla, liu2025hybridvla, qu2025spatialvla}.
However, misalignment among vision, language, and action modalities remains a core limitation, constraining adaptation to novel tasks and embodiments.
While alignment has been extensively explored in bi-modal vision–language models, its tri-modal extension is less understood. 
In this work, we aim to bridge this gap by introducing an alignment loss that explicitly enforces cross-modal consistency, enhancing both generalization and inference efficiency in VLA-based robotic manipulation.

\subsection{Generative Models for Policy Learning}

Recent progress in policy learning has increasingly cast decision making as a \emph{conditional sequence modeling} (CSM) problem \citep{DT, TT, GDT, QT, hu2026state}, leveraging Transformer architectures \citep{transformer}. In this paradigm, historical trajectories (e.g., state–action–reward tuples) are serialized as input sequences, and policies are obtained by predicting actions via goal-conditioned autoregressive decoding. This recasts offline reinforcement learning (RL) as supervised learning, alleviating optimization instabilities associated with bootstrapped value estimation \citep{srivastava2019training, kumar2019reward}. By conditioning on long temporal context rather than strict Markov structure, CSM policies can capture long-horizon dependencies and improve robustness across environments \citep{hu2024transforming}.

In parallel, generative modeling has become a central tool for policy learning, particularly diffusion models and flow-based methods. Diffusion models iteratively denoise latent variables to produce structured samples, and have been adopted in robotics to parameterize continuous, potentially multimodal action distributions \citep{hu2025analytic, hu2025solving}. Diffusion Policy \citep{chi2025diffusion} is a representative example, demonstrating strong expressivity in manipulation settings. More recently, flow matching has been proposed as an efficient alternative, offering smoother generation and faster training while maintaining competitive representational capacity \citep{chisari2024learning, zhang2024affordance}. These developments motivate our focus on flow-based formulations for VLA policy learning, aiming to improve both inference efficiency and cross-modal alignment.

%% file: sections/3_Method.tex
\section{Method}
\label{sec:med}

In this section, we introduce the AGFT: (i) a description of the architecture, instantiated as a scalable DiT with in-context conditioning; (ii) a definition of the objective used to train multi-modal action generation; and (iii) a specification of the inference pipeline.


\subsection{Architecture}

We build upon the Dita \citep{hou2025dita} architecture, which provides a simple yet flexible foundation for mapping multimodal inputs—language and visual observations—directly into continuous robot actions through noise-conditioned decoding. Dita receives both natural language instructions and third-person camera images as input. The language commands are tokenized using a frozen CLIP encoder \citep{radford2021learning}, while visual features are extracted through DINOv2 \citep{oquab2023dinov2}. Since DINOv2 is pre-trained on web-scale data rather than robot-specific observations, its parameters are jointly optimized with the policy network to better adapt to embodied data. To further improve efficiency, a Q-Former \citep{li2023blip} with FiLM-based conditioning \citep{perez2018film} selects context-relevant visual tokens by modulating DINOv2 features according to the instruction embedding. 
Each end-effector action is represented as a 7D vector composed of translation, rotation, and gripper components, and padded to match the dimensionality of multimodal tokens. During training, stochastic perturbations are applied only to these 7D action components, while language and visual tokens remain uncorrupted (details are provided in Appendix~\ref{appsec:model}).

As illustrated in Figure \ref{fig:architecture}, the model uses a Diffusion Transformer (DiT) backbone \citep{peebles2023scalable} as the sequence model for action-chunk generation.
Language, image, and temporal embeddings are concatenated within a causal Transformer structure, allowing action predictions to be directly conditioned on rich multimodal context.
Training optimizes a flow-matching objective in continuous action space: the Transformer parameterizes the conditional velocity (or flow) field that transports noisy action samples toward the data distribution. This design avoids auxiliary diffusion heads used in several prior multimodal policy models \citep{dasari2025ingredients,reuss2024multimodal,team2024octo,wen2025tinyvla}, while retaining the scalability of large Transformers and supporting straightforward extension to additional sensory modalities.

\subsection{Training Objective}

The proposed AGFT framework adopts the flow-matching formulation to enhance training efficiency and inference speed, while additionally incorporating a cross-modal alignment loss to strengthen consistency among visual, linguistic, and action representations.

Formally, the denoising network is parameterized as a causal Transformer $v_{\theta}(c_{\text{lang}}, c_{\text{img}}, t, z_t)$, where $c_{\text{lang}}$ and $c_{\text{img}}$ denote the language and visual embeddings, respectively. These embeddings are obtained from pre-trained encoders:
\begin{align}
    c_{\text{lang}} &= \mathrm{CLIP}(\text{language}), \\
    c_{\text{img}} &= \mathrm{DINOv2}(\text{image}),
\end{align}
and $t \sim \mathcal{U}(0,1)$ represents the continuous time step along the flow trajectory, while $z_t$ denotes the noisy action embedding at time $t \in [0,1]$.

Instead of performing iterative Gaussian denoising as in diffusion models, AGFT learns a deterministic vector field that directly transports samples from a source distribution $z_0 \sim \mathcal{N}(0, I)$ to the target action distribution $z_1 \sim D$ from the given trajectories. To simplify notation, we set \(\mathbb{E}[\cdot]:=\mathbb{E}_{\,z_0 \sim \mathcal{N}(0,I),\, z_1 \sim D,\, c_{\mathrm{img}},\, c_{\mathrm{lang}}}[\cdot].\)
 The corresponding flow-matching loss \citep{chisari2024learning} is defined as:
\begin{equation}
\label{eq:fm_loss}
L_{\mathrm{FM}}(\theta) = \mathbb{E} \Big[\lVert v_{\theta}(c_{\text{img}}, c_{\text{lang}}, t, z_t) - (z_1 - z_0) \rVert_2^2\Big],
\end{equation}
where $v_{\theta}$ predicts the instantaneous velocity along the data flow between $z_0$ and $z_1$. This formulation removes the need for stochastic denoising iterations, substantially accelerating inference while maintaining trajectory smoothness.

To mitigate the tri-modal misalignment inherent in VLA models, we introduce a contrastive alignment objective that encourages semantic coherence among action, visual, and linguistic representations. 
While $v_{\theta}$ governs the flow dynamics, the same Transformer backbone produces final-layer embeddings prior to the projection head:
\begin{align}
    e_{\text{act}} &= f_{\theta}(c_{\text{img}}, c_{\text{lang}}, t, z_t).
\end{align}
We employ a pairwise contrastive objective of the form:
\begin{align}
\label{eq:contraloss}
L_{\mathrm{cons}}(x_1, &x_2)
=  
Y \, D^2(x_1, x_2) + (1 - Y) \max(m - D(x_1, x_2), 0)^2,
\end{align}
where $D(x_1, x_2) = \lVert x_1 - x_2 \rVert_2$ measures the L2 distance between embeddings, $Y \in \{0,1\}$ indicates whether $(x_1, x_2)$ forms a semantically aligned positive pair, and $m > 0$ is a margin hyperparameter. 
We label a pair as positive ($Y=1$) if the two embeddings come from the \emph{same demonstration trajectory at the same time step}, and as negative ($Y=0$) otherwise.
This alignment loss is computed across triplet pairs: $(c_{\text{img}}, e_{\text{act}})$, $(c_{\text{lang}}, e_{\text{act}})$, and $(c_{\text{img}}, c_{\text{lang}})$, thereby enforcing cross-modal consistency among all modalities.
At each control timestep $t$, the token counts of the language, image, and action representations ($c_{\text{lang}}$, $c_{\text{img}}$, and $e_{\text{act}}$) generally differ. To achieve consistent cross-modal conditioning, we mean-pool the valid tokens of $c_{\text{lang}}$ and $c_{\text{img}}$ to obtain a per-timestep contextual vector, which is subsequently projected into the action embedding space. This produces a temporally aligned, fixed-dimensional conditioning signal that promotes geometric and semantic consistency between the multimodal context and the corresponding action representation.

The final optimization objective integrates the flow matching and alignment components:
\begin{align}
\label{eq:final_loss}
L_{final} = L_{FM}(\theta)
+ \alpha \big( L_{\mathrm{cons}}(c_{\text{lang}}, e_{\text{act}}) 
+ L_{\mathrm{cons}}(c_{\text{img}}, e_{\text{act}}) + L_{\mathrm{cons}}(c_{\text{img}}, c_{\text{lang}}) \big),
\end{align}
where $\alpha$ controls the relative contribution of the alignment term.  
Together, these objectives allow AGFT to produce action embeddings that are both efficient to infer and semantically grounded in accompanying visual and linguistic contexts, thereby improving generalization and embodiment adaptability in robotic manipulation tasks.

\subsection{Inference}

During inference, the AGFT performs deterministic trajectory generation by integrating along the learned flow field instead of relying on iterative stochastic denoising as in diffusion-based methods. Given the learned velocity predictor $v_{\theta}(c_{img}, c_{lang}, t, z_t)$, the action sequence is recovered by solving the following ordinary differential equation (ODE):
\begin{equation}
\frac{dz_t}{dt} = v_{\theta}(c_{img}, c_{lang}, t, z_t),
\end{equation}
where the latent state $z_t$ evolves from a base distribution $z_0 \sim \mathcal{N}(0, I)$ toward the target action $z_1$ under the policy's learned flow field. 

The deterministic nature of the flow process preserves trajectory smoothness and predictability, which are crucial for stable manipulation tasks. Conditioned jointly on visual and linguistic contexts, AGFT thus produces coherent and semantically consistent actions in real time, making it highly suitable for adaptive and efficient robotic policy deployment.


\section{Theoretical Support}
\label{sec:theoretic}

In this section, we provide theoretical justification for the proposed AGFT framework by deriving upper and lower bounds for the flow-matching loss, $L_{\mathrm{FM}}$.

To simplify notation, we set \(\Delta := z_1 - z_0\) and \(\mu := \tfrac13\big(c_{\mathrm{img}}+c_{\mathrm{lang}}+e_{\mathrm{act}}\big)\).
The symbol $\stackrel{d}{=}$ denotes equality in distribution. We also write \(v_\theta(c_{\mathrm{img}},c_{\mathrm{lang}},t,z_t)=h_\theta(e_{\mathrm{act}})\), where \(h_\theta\) denotes the final linear projection head in the Transformer architecture.
We define the average pairwise squared distance among \(c_{\mathrm{img}},c_{\mathrm{lang}},e_{\mathrm{act}}\) as follows:
\[
d_\mathrm{avg} \;:=\; \tfrac13\!\big(\|e_{\mathrm{act}}-c_{\mathrm{img}}\|_2^2
+\|e_{\mathrm{act}}-c_{\mathrm{lang}}\|_2^2
+\|c_{\mathrm{img}}-c_{\mathrm{lang}}\|_2^2\big).
\]

\begin{assumption}[$L_h$-Lipschitz]\label{Assumption: Li-Lipschitz}
There exists a constant $L_h>0$ such that for all $u,v$ in the domain of $h_\theta$,
\[
\|h_\theta(u)-h_\theta(v)\|_2 \;\le\; L_h\,\|u-v\|_2.
\]
\end{assumption}

\begin{assumption}[$m_h$-co-Lipschitz]\label{Assumption: co-Lipschitz}
There exists a constant $m_h>0$ such that for all $u,v$ in the domain of $h_\theta$,
\[
\|h_\theta(u)-h_\theta(v)\|_2 \;\ge\; m_h\,\|u-v\|_2.
\]
\end{assumption}

\begin{assumption}[Conditional exchangeability]\label{Assumption: Conditional exchangeability}
Let $X_1,\dots,X_n$ be observable random variables and let $\mu$ be a conditioning variable.
We say that $(X_1,\dots,X_n)$ are conditionally exchangeable given $\mu$ if, for every permutation $\pi$ of $\{1,\dots,n\}$, we have $(X_1,\dots,X_n)\mid\mu \stackrel{d}{=} (X_{\pi(1)},\dots,X_{\pi(n)})\mid\mu$.
\end{assumption}


Assumptions~\ref{Assumption: Li-Lipschitz}--\ref{Assumption: co-Lipschitz} are imposed only on the final projection head $h_\theta$.
Since $h_\theta(x)=Wx$, Assumption~\ref{Assumption: Li-Lipschitz} holds with $L_h=\|W\|_{\mathrm{op}}$.
Assumption~\ref{Assumption: co-Lipschitz} characterizes an idealized regime where the projection head $h_\theta$ is bounded away from zero on the support of the input data. 
This allows us to adopt an effective constant $m_h=\sigma_{\min}(W) > 0$, which is crucial for rigorously deriving the lower bounds of the learning objective $\mathcal{L}_{\mathrm{FM}}(\theta)$.
We use this assumption only to derive a quantitative lower bound.

Assumption~\ref{Assumption: Conditional exchangeability} posits a statistical symmetry across the three modality-specific embeddings.
Abstracting away the inherent heterogeneity of the raw modalities, this assumption implies that, conditioned on the shared semantic signal $\mu$, the embeddings $c_{\mathrm{img}}, c_{\mathrm{lang}}, e_{\mathrm{act}}$ follow identical conditional distributions.
This simplification allows us to treat the information contribution from each modality as mathematically interchangeable, facilitating the derivation of a lower bound that explicitly depends on $d_{\mathrm{avg}}$.

\begin{theorem}\label{Theorem: upper and lower bounds}
Under Assumption~\ref{Assumption: Li-Lipschitz}, we have
\[
L_{\mathrm{FM}}(\theta)
\;\le\;
2\,\mathbb{E}\,\|h_\theta(\mu)-\Delta\|_2^2
\;+\;2L_h^{2}\,\mathbb{E}[d_{\mathrm{avg}}].
\]
Moreover, under Assumptions~\ref{Assumption: co-Lipschitz} and~\ref{Assumption: Conditional exchangeability}, let
\(u^\star(\mu):=\mathbb{E}[\Delta\mid \mu]\) be the Bayes target when only \(\mu\) is observed. Then
\[
L_{\mathrm{FM}}(\theta)
\;\ge\;
\mathbb{E}\big\|u^\star(\mu)-\Delta\big\|^2
\;+\;\frac{m_h^{2}}{3}\,\mathbb{E}[d_{\mathrm{avg}}].
\]
\end{theorem}


\begin{remark}
The proof of Theorem~\ref{Theorem: upper and lower bounds} is deferred to Appendix~\ref{appsec:theoretical}.
Theorem~\ref{Theorem: upper and lower bounds} isolates how the alignment statistic $d_{\mathrm{avg}}$ enters the flow-matching objective:
$\mathbb{E}[d_{\mathrm{avg}}]$ appears in the upper bound with coefficient $2L_h^2$ and in the lower bound with coefficient $m_h^2/3$.
Thus, reducing $\mathbb{E}[d_{\mathrm{avg}}]$ narrows the admissible range of $L_{\mathrm{FM}}(\theta)$ around the conditional Bayes risk $\mathbb{E}\|u^\star(\mu)-\Delta\|^2$.
This provides an objective-level rationale for penalizing misalignment via $L_{\mathrm{cons}}$.
\end{remark}

%% file: sections/4_Exp.tex
\section{Experiments}
\label{sec:exp}

\subsection{Benchmarks and Baselines}

\textbf{Benchmarks.} We evaluate our proposed AGFT on the LIBERO benchmark \citep{liu2023libero}, a comprehensive suite for studying embodied robotic manipulation. 
LIBERO comprises four task collections—LIBERO-Spatial, LIBERO-Object, LIBERO-Goal, and LIBERO-Long—each containing ten tasks and 500 expert demonstrations. These suites collectively evaluate spatial reasoning, object-centric manipulation, goal-oriented behavior, and long-horizon task execution. 
For each task, the policy observes RGB inputs from both a third-person and a wrist-mounted camera, accompanied by natural language instructions and the end-effector state. 
To isolate multimodal understanding capabilities, we intentionally exclude depth, affordance, or other auxiliary sensory inputs, ensuring that the model learns purely from visual and linguistic signals.
We employ the modified version of LIBERO from OpenVLA \citep{kim2024openvla} as the data source for finetuning and evaluation.
Additional results on CALVIN \citep{mees2022calvin} are deferred to Appendix \ref{ap:add_exp}. Owing to space constraints, we primarily report LIBERO performance in the main text and adopt it as the principal benchmark for our ablation studies.

\textbf{Baselines.} We compare AGFT against a diverse set of state-of-the-art policies across the two dominant paradigms of action generation: autoregressive (AR) token-based models and continuous generative models. 
The AR category includes OpenVLA \citep{kim2024openvla}, Octo \citep{team2024octo}, $\pi_0$+FAST \citep{black2024pi_0, pertsch2025fast}, CoTVLA \citep{zhao2025cot}, TraceVLA \citep{zheng2024tracevla}, SpatialVLA \citep{qu2025spatialvla}, and WorldVLA \citep{cen2025worldvla}, which model actions as discrete token sequences decoded from vision–language embeddings.
For continuous generative policies, we evaluate Diffusion Policy \citep{chi2025diffusion}, MDT \citep{reuss2024multimodal}, and Dita \citep{hou2025dita}, alongside our AGFT. 
All methods are trained and evaluated under identical observation modalities and assessed using the official LIBERO success rate metric. Reported baseline scores are either drawn from original publications or reproduced using publicly available implementations to ensure fair comparison.

\begin{table*}[!t]
\rowcolors{2}{gray!10}{white}
\centering
\caption{Evaluation results on LIBERO, measured by success rate (\%). Best in \textbf{bold}.}
\vspace{-.2cm}
\label{tab:libero-results}
\scalebox{0.8}{
\begin{tabular}{L{2.5cm}C{2.1cm}C{2.1cm}C{2.2cm}C{2.3cm}C{2.2cm}C{1cm}}
\toprule
\multicolumn{1}{l}{\textbf{Method}} &
\multicolumn{1}{c}{\textbf{LIBERO-Sp}} &
\multicolumn{1}{c}{\textbf{LIBERO-Obj}} &
\multicolumn{1}{c}{\textbf{LIBERO-Goal}} &
\multicolumn{1}{c}{\textbf{LIBERO-Long}} &
\multicolumn{1}{c}{\textbf{Average}} &
\multicolumn{1}{c}{\textbf{Year}}\\
\midrule
OpenVLA          & 84.7 & 88.4 & 79.2 & 53.7 & 76.5 & 2024 \\
Octo            & 78.9 & 85.7 & 84.6 & 51.1 & 75.1 & 2024 \\
$\pi_{0}$ + FAST & \textbf{96.4} & 96.8 & 88.6 & 60.2 & 85.5 & 2025\\
CoTVLA         & 81.1 & 87.5 & \textbf{91.6} & \textbf{87.6} & 87.0 & 2025\\
TraceVLA & 84.6 &  85.2 & 75.1 & 54.1 & 74.8 & 2024\\
SpatialVLA & 88.2 & 89.9 & 78.6 & 55.5 & 78.1 & 2025\\
WorldVLA         & 87.6 & 96.2 & 83.4 & 60.0 & 81.8 & 2025\\
Diffusion Policy  & 78.3 & 92.5 & 68.3 & 50.5 & 72.4 & 2025\\
MDT             & 78.5 & 87.5 & 73.5 & 64.8 & 76.1 & 2024\\ \midrule
Dita      & 84.2 & 96.3 & 85.4 & 63.8 & 82.4 & 2025\\
Dita + Align Loss & 90.2 & 96.4 & 90.0 & 72.3 & 87.2 & 2026 \\
\textbf{AGFT} & 92.4 & \textbf{97.6} & 90.2 & 75.2 & \textbf{88.9} & 2026\\
\bottomrule
\end{tabular}}
\vspace{-.6cm}
\end{table*}

\subsection{Evaluation Performance}

We evaluate AGFT in the simulated LIBERO environment to assess its generalization and multimodal reasoning capabilities. Table \ref{tab:libero-results} summarizes the quantitative performance across the four task suites. AGFT achieves success rates of 92.4\%, 97.6\%, 90.2\%, and 75.2\% on LIBERO-Spatial, -Object, -Goal, and -Long respectively, yielding an overall average of 88.9\%, surpassing all baselines.
Compared to Dita—the architectural backbone of AGFT—augmenting Dita with the proposed alignment loss yields a 5.8\% relative improvement, while AGFT attains a further 7.9\% relative gain in average success rate, substantiating the effectiveness of the contrastive alignment objective. The gains are especially pronounced on the LIBERO-Long suite, where AGFT exhibits improved stability and semantic consistency over long-horizon executions, reflecting stronger temporal coherence and cross-modal grounding.

Notably, AGFT also surpasses all autoregressive baselines, including $\pi_0$+FAST and OpenVLA, suggesting that alignment-driven flow-based modeling achieves a favorable balance between task adaptation and inference efficiency.
Overall, these results suggest that AGFT effectively captures the fine-grained visual cues required for temporally extended manipulation and generalizes robustly across scenes. We attribute these gains to explicit cross-modal alignment, which strengthens the coupling among visual, linguistic, and action representations, thereby promoting more consistent action selection and improved downstream task adaptation.

\subsection{Ablations}

\begin{table*}[t!]
\rowcolors{2}{gray!10}{white}
\centering
\caption{Ablation results on LIBERO-Long. Starting from DiT fine-tuning, we report success rate (\%) gains by incrementally enabling Flow Matching (FM) and contrastive alignment terms (Text–Action, Image–Action, and Text–Image). \ding{51} denotes the component is active.}
\vspace{-.2cm}
\label{tab:module}
\small
\centering
\scalebox{0.7}{
\begin{tabular}{C{2.5cm}C{2.5cm}C{2.5cm}C{2.5cm}C{2.5cm}C{2.5cm}}
\toprule[2pt]
Index & Flow Matching & Text-Action Loss &  Image-Action Loss & Text-Image Loss & Performance \\
\midrule 
1 & & & & & 63.8 \\
2 & & \ding{51} & & & 70.2 \\
3 & & & \ding{51} & & 71.6 \\
4 & & & & \ding{51} & 64.1 \\
5 & & \ding{51} & \ding{51} & & 71.8 \\
6 & & \ding{51} & \ding{51} & \ding{51} & 72.3 \\
7 & \ding{51} & & & & 67.8  \\
8 & \ding{51} & \ding{51} & & & 72.4  \\
9 & \ding{51} & & \ding{51} &  & 73.4 \\
10 & \ding{51} &  &  & \ding{51}  & 68.0 \\
11 & \ding{51} & \ding{51} & \ding{51} &  & 74.6  \\
12 & \ding{51} & \ding{51} & \ding{51} & \ding{51}  & 75.2  \\
\bottomrule[2pt]
\end{tabular}
}
\vspace{-.5cm}
\end{table*}

\paragraph{Components.} To better understand the contribution of each component within AGFT, we conduct a detailed ablation study on the LIBERO-Long suite, which is particularly sensitive to long-horizon temporal consistency and multimodal grounding. As summarized in Table \ref{tab:module}, we evaluate various combinations of the flow matching module and three contrastive alignment objectives: text–action, image–action, and text–image losses.

The baseline configuration (Index 1), which corresponds to the original Dita objective, achieves a success rate of 63.8\%. Introducing either the text–action or image–action contrastive loss individually leads to substantial gains, reaching 70.2\% and 71.6\% respectively. This demonstrates that aligning actions with their corresponding modality-specific features yields more coherent and grounded policy execution. When both are jointly applied (Index 5–6), the performance rises further to over 72\%, confirming the complementary nature of textual and visual alignment in shaping a consistent action manifold.

Replacing diffusion supervision with the flow matching objective (Index 7) already improves the baseline to 67.8\%, validating that deterministic flow-based training enables smoother and more efficient policy learning. Combining flow matching with either alignment term yields additional improvements (Indices 8–11), while the full configuration integrating flow matching with all three contrastive losses (Index 12) achieves the highest success rate of 75.2\%. Notably, the text–image contrastive loss contributes a more modest performance gain compared to the action-centered objectives, suggesting that action alignment serves as the dominant factor in enhancing cross-modal consistency.



\paragraph{Contrastive Loss.} To evaluate the impact of different alignment objectives, we compare our pairwise contrastive loss (Equation \ref{eq:contraloss}) with the widely adopted InfoNCE formulation \citep{liu2023contrastive}. The InfoNCE loss is defined as:
\begin{align}
    \epsilon(x_i, x_j) &= 
    \frac{x_i^{\top} x_j}{\lVert x_i \rVert_2 \, \lVert x_j \rVert_2 \, \tau}, \\
    L_{\mathrm{InfoNCE}} &= 
    -\frac{1}{N} \sum_{i=1}^{N} 
    \log
    \frac{
    \exp\big(\epsilon(x_i, x_i^+)\big)
    }{
    \sum_{j=1}^{N} \exp\big(\epsilon(x_i, x_j)\big)
    }, \label{eq:contras2}
\end{align}
where $\tau > 0$ is a temperature hyperparameter that controls the sharpness of the similarity distribution, and $(x_i, x_i^+)$ denotes a positive (aligned) pair among $N$ sampled instances.  
Unlike InfoNCE, which enforces feature discriminability through softmax normalization over batch samples, our pairwise margin-based formulation directly penalizes misaligned embeddings via Euclidean distance constraints. This explicit margin mechanism provides a more balanced treatment of positive and negative pairs, yielding improved regularization for cross-modal consistency within the multimodal action representation space.


As reported in Figure \ref{fig:loss}, the pairwise contrastive loss consistently outperforms InfoNCE across most LIBERO suites, achieving +1.8\%, +1.6\%, and +3.2\% higher success rates on LIBERO-Spatial, LIBERO-Object, and LIBERO-Goal tasks, respectively. The only marginal exception occurs in LIBERO-Long, where InfoNCE yields a slight advantage (+0.8\%). These results highlight that margin-based pairwise supervision provides stronger local structure learning, reinforcing one-to-one alignment between actions and corresponding visual–linguistic signals. Conversely, InfoNCE’s reliance on batch-level negatives may dilute supervision signals in small-sample embodied datasets.
Nonetheless, the inclusion of the InfoNCE-based alignment objective still yields improvements over the original Dita baseline, suggesting that explicitly enforcing multimodal alignment consistently enhances overall policy performance.

\begin{figure}[t] 
    \centering
    \begin{subfigure}[b]{0.48\textwidth}
        \centering
        \includegraphics[width=\linewidth]{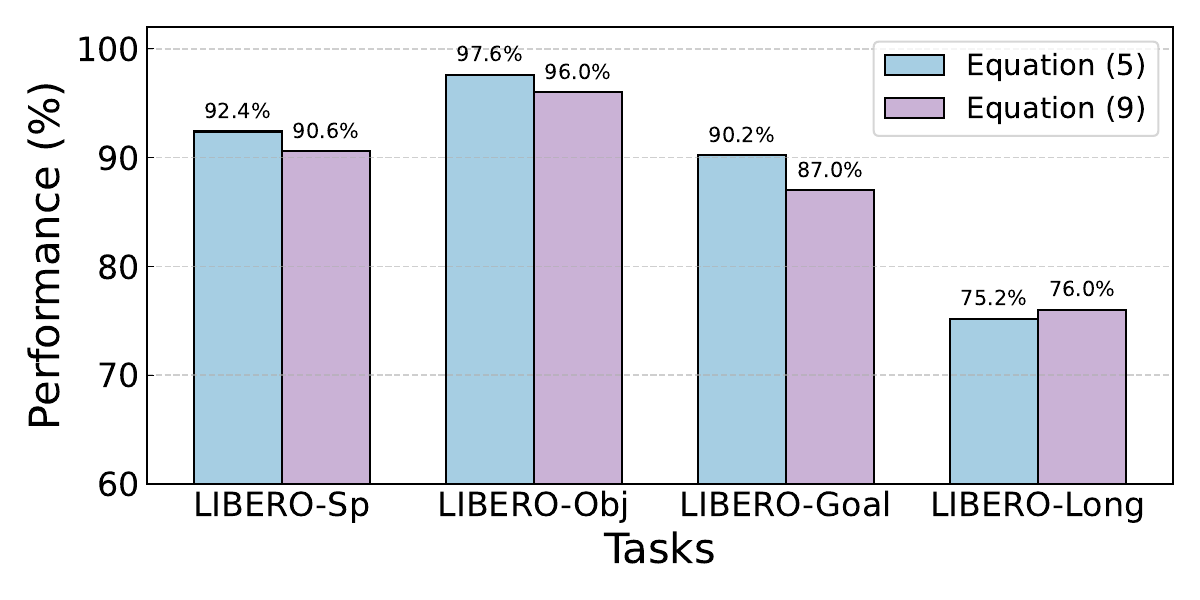}
        \vspace{-.6cm}
        \caption{Ablation study on alignment losses.}
        \label{fig:loss}
    \end{subfigure}
    \hfill 
    \begin{subfigure}[b]{0.48\textwidth}
        \centering
        \includegraphics[width=\linewidth]{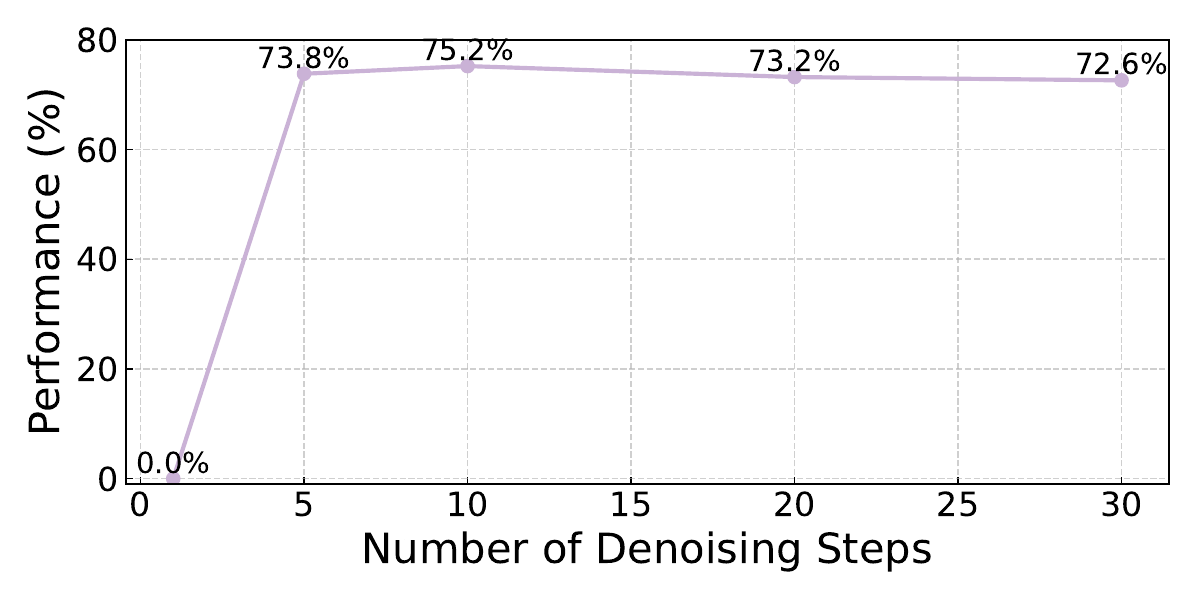}
        \vspace{-.6cm}
        \caption{Performance across denoising steps.}
        \label{fig:denoise}
    \end{subfigure}
    \vspace{-.2cm}
    \caption{Experimental analysis on LIBERO-Long and diverse tasks. (a) compares different alignment losses, while (b) shows the impact of denoising steps.}
    \vspace{-.6cm}
    \label{fig:combined_analysis}
\end{figure}


\paragraph{Denoising Steps.} Diffusion-based generative models typically rely on a large number of stochastic denoising steps to iteratively reconstruct signals from noise, as demonstrated in image synthesis domains \citep{rombach2022high}. When adapted for robotic control, this process directly affects inference latency and thus the real-time feasibility of the policy. Specifically, each denoising iteration introduces additional computational overhead, making it essential to balance sample quality with control frequency.
In contrast, our method replaces iterative stochastic denoising with a flow matching formulation, which learns a continuous and deterministic transport field between the latent base distribution and target action space. This design enables AGFT to generate actions with substantially fewer integration steps while retaining comparable fidelity.

To empirically examine this benefit, we conduct an ablation study on the LIBERO-Long benchmark, varying the number of denoising (integration) steps during inference. As shown in Figure \ref{fig:denoise}, AGFT achieves near-optimal performance with as few as five steps, yielding a 73.8\% success rate—closely matching the 75.2\% obtained at ten steps. Increasing the number of steps beyond this threshold provides negligible gain and even leads to minor degradation, likely due to redundant updates and the accumulation of integration noise. Notably, using a single step (no iterative refinement) fails completely, confirming that minimal integration is necessary to trace the learned flow trajectory. These findings highlight that the flow-based formulation dramatically improves inference efficiency, enabling high-quality action generation with an order of magnitude fewer steps than traditional diffusion models. A detailed comparison of inference times during the inference is provided in Appendix \ref{appsec:model}.

\begin{wraptable}{r}{0.5\textwidth} 
    \vspace{-12pt} 
    \centering
    \small
    \caption{Success rate (\%) on LIBERO-Long. \ding{51} indicates the parameter subset is optimized.}
    \label{tab:parameters}
    \rowcolors{2}{gray!10}{white}
    \scalebox{0.8}{ 
        \begin{tabular}{cccc} 
        \toprule[1.5pt]
        CLIP & DinoV2 & Causal Trans. & LIBERO-L \\
        \midrule 
         & & \ding{51} & 55.0 \\
        \ding{51} & & \ding{51} & 51.4  \\
         & \ding{51} & \ding{51} & 75.2 \\
        \ding{51} & \ding{51} & \ding{51} & 74.2  \\
        \bottomrule[1.5pt]
        \end{tabular}
    }
    \vspace{-10pt} 
\end{wraptable}

\paragraph{Trainable Parameters.} We further analyze the influence of trainable parameter selection on overall policy performance, focusing on the effect of fine-tuning different subsystem components within AGFT. In our primary experiments, only the Causal Transformer and its associated vision networks are fine-tuned, while the pretrained language encoder parameters remain frozen. To evaluate the trade-off between adaptation flexibility and representation stability, we systematically vary the set of trainable modules and assess their impact on the LIBERO-Long benchmark tasks.

As summarized in Table \ref{tab:parameters}, fine-tuning only the Causal Transformer yields a baseline success rate of 55.0\%. When fine-tuning is extended to the CLIP language encoder, performance decreases to 51.4\%, indicating that adapting the pre-trained linguistic embeddings can disrupt the alignment between visual and text representations, likely due to distributional mismatch between general language data and task-specific robotic instructions. In contrast, unfreezing the DINOv2 visual encoder alongside the Causal Transformer significantly improves performance to 75.2\%, suggesting that moderate adaptation of visual representations enhances grounding precision and spatial understanding. Interestingly, jointly fine-tuning both CLIP and DINOv2 slightly reduces performance to 74.2\%, supporting the hypothesis that while visual adaptability benefits manipulation reasoning, excessive adjustment of language parameters tends to destabilize cross-modal alignment.

%% file: sections/5_Con.tex
\vspace{-.2cm}
\section{Conclusion}
\vspace{-.1cm}
\label{sec:con}

We propose Alignment-Guided Flow Transformer (AGFT), a unified vision–language–action framework for reducing tri-modal misalignment in embodied policy learning. By integrating contrastive alignment with deterministic flow matching, AGFT yields coherent multimodal representations and stable, efficient action generation.
Our theoretical analysis shows alignment regularization tightens flow-matching bounds, and extensive experimental results verify consistent gains over diffusion-based DiT fine-tuning baselines while remaining competitive with strong VLA baselines.

%% file: sections/6_suppl.tex
\paragraph{Limitation.}
While we provide extensive experiments, ablations, and theoretical analysis, our empirical validation is currently restricted to the LIBERO and CALVIN benchmark, which may not fully capture the diversity of real-world embodied settings. Nevertheless, the theoretical guarantees are benchmark-agnostic, and the observed improvements are consistent across tasks and ablations, offering clear evidence of the method’s effectiveness—particularly the contribution of the alignment loss.

\section{Model and Training Scheme}
\label{appsec:model}

\paragraph{Model.} Language instructions are encoded by a pretrained CLIP text encoder \citep{radford2021learning}, which remains frozen throughout training to preserve semantic consistency. Visual observations are resized to $224 \times 224$ and processed using a pretrained DINOv2 backbone \citep{oquab2023dinov2}, whose parameters are fully finetuned to adapt visual features to downstream robotic tasks. A Q-Former, trained from scratch with a depth of 4 and 32 latent queries, projects high-dimensional visual features into a compact representation, while text tokens are injected into each block via FiLM conditioning to infuse linguistic cues into the visual stream. Action inputs are perturbed using a DDPM noise scheduler with 100 timesteps, and each timestep index is embedded via sinusoidal positional encoding. These multimodal features are fused within a causal Transformer, following a LLaMA2-style architecture \citep{touvron2023llama} with 12 self-attention layers and a hidden size of 768. The Transformer predicts the injected noise at each diffusion step. In total, the model contains 334M parameters, of which 221M are trainable. 

\paragraph{Pretraining.} The proposed model integrates visual, linguistic, and action modalities within a unified causal Transformer architecture \citep{hou2025dita}. During pretraining, we adopt the pipeline introduced by \citet{hou2025dita} using the Open X-Embodiment (OXE) datasets \citep{kim2024openvla, o2024open}, following prior works \citep{team2024octo, kim2024openvla} in terms of dataset composition and weighting. All action trajectories are normalized and filtered as in \citet{o2024open}. Pretraining employs a DDPM-based diffusion objective \citep{ho2020denoising, song2020denoising, hu2025prompt} with $T_{\text{train}} = 1000$ timesteps. Training is conducted with the AdamW optimizer \citep{loshchilov2017decoupled} for 100{,}000 steps, using learning rates of $1\times10^{-4}$ for both the causal Transformer and Q-Former, and $1\times10^{-5}$ for DINOv2. 

\paragraph{Finetuning.} During finetuning, the network is optimized for 40{,}000 steps using the AdamW optimizer with a base learning rate of $1\times10^{-4}$ across all task variants. A half-cycle cosine scheduler is applied to smoothly decay the learning rate during training, promoting stable convergence. Training is performed with a total batch size of 96, distributed across four NVIDIA RTX A6000 GPUs. To balance the influence of the contrastive alignment term, we set the contrastive loss weight to $1.0$ for LIBERO-Object and $0.001$ for LIBERO-Spatial, LIBERO-Goal, LIBERO-Long, and CALVIN tasks, reflecting the varying degrees of cross-modal supervision required among different task categories.

\paragraph{Inference Time.} We evaluate the inference efficiency of AGFT by comparing its action generation latency against Dita. AGFT generates actions in 0.295s per sample (with the inference step set to 10), whereas Dita—which depends on multi-step diffusion-based denoising—requires 1.573s (with 100 inference steps, following the default configuration in \citet{hou2025dita}). This pronounced reduction in inference time highlights the effectiveness of AGFT’s deterministic flow-matching formulation, enabling significantly faster and more computationally efficient policy execution.

\section{Theoretical Support}
\label{appsec:theoretical}
Before proving Theorem~\ref{Theorem: upper and lower bounds}, we record two auxiliary lemmas, stated as Lemma~\ref{Lemma: Var_Y > mVar_X} and Lemma~\ref{Lemma: mean under exchange}.

\begin{lemma}\label{Lemma: Var_Y > mVar_X}
Let $Y := h(X)$, where $h$ satisfies the $m_h$-co-Lipschitz condition and
$\mathbb{E}\!\left[\|X\|_2^{2}\mid \mu\right]<\infty$.
Then
\[
\mathrm{Var}(Y\mid \mu)\ \ge\ m_h^{2}\,\mathrm{Var}(X\mid \mu)\qquad\text{a.s.}
\]
\end{lemma}

\begin{proof}
Let $X'$ be an independent copy of $X$ given $\mu$, and set $Y' := h(X')$.
Using the identity
\[
\mathrm{Var}(U\mid \mu)\;=\;\tfrac12\,\mathbb{E}\!\left[\|U-U'\|_2^{2}\,\middle|\,\mu\right],
\]
valid for any square-integrable $U$ with $U'$ an independent conditional copy, we obtain
\[
\mathrm{Var}(Y\mid \mu)
= \tfrac12\,\mathbb{E}\!\left[\|Y-Y'\|_2^{2}\,\middle|\,\mu\right]
\;\ge\; \tfrac12\,\mathbb{E}\!\left[m_h^{2}\|X-X'\|_2^{2}\,\middle|\,\mu\right]
= m_h^{2}\,\mathrm{Var}(X\mid \mu),
\]
where the inequality follows from the $m_h$-co-Lipschitz property of $h$.
\end{proof}

\begin{lemma}\label{Lemma: mean under exchange}
Under Assumption~\ref{Assumption: Conditional exchangeability}, with $\mu=\tfrac13(c_{\mathrm{img}}+c_{\mathrm{lang}}+e_{\mathrm{act}})$, we have
\[
\mathbb{E}\big[e_{\mathrm{act}}\mid \mu \big]=\mu
\quad\text{and}\quad
\mathbb{E}\,\|e_{\mathrm{act}}-\mu\|_2^2
=\frac{1}{9}\,\mathbb{E}\!\Big[\|e_{\mathrm{act}}-c_{\mathrm{img}}\|_2^2+\|e_{\mathrm{act}}-c_{\mathrm{lang}}\|_2^2+\|c_{\mathrm{img}}-c_{\mathrm{lang}}\|_2^2\Big].
\]
\end{lemma}

\begin{proof}
By conditional exchangeability,
\[
\mathbb{E}\big[e_{\mathrm{act}}\mid \mu \big]
=\mathbb{E}\big[c_{\mathrm{img}}\mid \mu \big]
=\mathbb{E}\big[c_{\mathrm{lang}}\mid \mu \big].
\]

Averaging these three equalities gives
\[
\mu=\tfrac13\Big(\mathbb{E}[c_{\mathrm{img}}\mid\mu]+\mathbb{E}[c_{\mathrm{lang}}\mid\mu]+\mathbb{E}[e_{\mathrm{act}}\mid\mu]\Big)
=\mathbb{E}[e_{\mathrm{act}}\mid\mu].
\]

For the second claim, the three point variance identity yields
\[
\|c_{\mathrm{img}}-\mu\|_2^2+\|c_{\mathrm{lang}}-\mu\|_2^2+\|e_{\mathrm{act}}-\mu\|_2^2
=\tfrac13\Big(\|e_{\mathrm{act}}-c_{\mathrm{img}}\|_2^2+\|e_{\mathrm{act}}-c_{\mathrm{lang}}\|_2^2+\|c_{\mathrm{img}}-c_{\mathrm{lang}}\|_2^2\Big).
\]

Conditioning on $\mu$ and using exchangeability, the three terms on the left have the same conditional expectation, hence
\[
\mathbb{E}\big[\|e_{\mathrm{act}}-\mu\|_2^2\mid\mu\big]
=\frac{1}{9}\,\mathbb{E}\!\Big[\|e_{\mathrm{act}}-c_{\mathrm{img}}\|_2^2+\|e_{\mathrm{act}}-c_{\mathrm{lang}}\|_2^2+\|c_{\mathrm{img}}-c_{\mathrm{lang}}\|_2^2\,\Bigm|\,\mu\Big].
\]

Taking expectation over $\mu$ gives the desired equality.
\end{proof}

We now present the proof of Theorem~\ref{Theorem: upper and lower bounds}.

\begin{proof}
(i) We first establish an upper bound on $L_{\mathrm{FM}}(\theta)$.
Using $\|a+b\|_2^2 \le 2\|a\|_2^2 + 2\|b\|_2^2$ with
$a = h_\theta(e_{\mathrm{act}})-h_\theta(\mu)$ and $b = h_\theta(\mu)-\Delta$, we obtain
\[
\|v_\theta(c_{\mathrm{img}},c_{\mathrm{lang}},t,z_t)-\Delta\|_2^2
= \|h_\theta(e_{\mathrm{act}})-\Delta\|_2^2
\le 2\|h_\theta(e_{\mathrm{act}})-h_\theta(\mu)\|_2^2
  + 2\|h_\theta(\mu)-\Delta\|_2^2.
\]

By $L_h$-Lipschitzness,
$\|h_\theta(e_{\mathrm{act}})-h_\theta(\mu)\|_2 \le L_h\,\|e_{\mathrm{act}}-\mu\|_2$.

For $\mu=\tfrac13\big(c_{\mathrm{img}}+c_{\mathrm{lang}}+e_{\mathrm{act}}\big)$, the three-point variance identity yields
\[
\|c_{\mathrm{img}}-\mu\|_2^2+\|c_{\mathrm{lang}}-\mu\|_2^2+\|e_{\mathrm{act}}-\mu\|_2^2
=\tfrac13\Big(\|e_{\mathrm{act}}-c_{\mathrm{img}}\|_2^2+\|e_{\mathrm{act}}-c_{\mathrm{lang}}\|_2^2+\|c_{\mathrm{img}}-c_{\mathrm{lang}}\|_2^2\Big).
\]

Since $\|c_{\mathrm{img}}-\mu\|_2^2$ and $\|c_{\mathrm{lang}}-\mu\|_2^2$ are nonnegative, we obtain
\[
\|e_{\mathrm{act}}-\mu\|_2^2
\;\le\; \tfrac13\!\Big(\|e_{\mathrm{act}}-c_{\mathrm{img}}\|_2^2
+\|e_{\mathrm{act}}-c_{\mathrm{lang}}\|_2^2
+\|c_{\mathrm{img}}-c_{\mathrm{lang}}\|_2^2\Big).
\]

Combining the displays and taking expectations yields 
\[
L_{\mathrm{FM}}(\theta)\;\le\;2\|h_\theta(\mu)-\Delta\|_2^2+\tfrac13\!\Big(\|e_{\mathrm{act}}-c_{\mathrm{img}}\|_2^2
+\|e_{\mathrm{act}}-c_{\mathrm{lang}}\|_2^2
+\|c_{\mathrm{img}}-c_{\mathrm{lang}}\|_2^2\Big).
\]

(ii) We then establish a lower bound for $L_{\mathrm{FM}}(\theta)$. Write \(A:=h_\theta(e_{\mathrm{act}})\) and \(B:=\Delta\). By the law of total expectation and conditional independence,
\begin{equation*}
\begin{aligned}
    \mathbb{E}\|A-B\|^2
=& \mathbb{E}\big[\mathrm{Var}(A\mid \mu)\big]
  + \mathbb{E}\big[\mathrm{Var}(B\mid \mu)\big]
  + \mathbb{E}\big\|\,\mathbb{E}[A\mid \mu]-\mathbb{E}[B\mid \mu]\,\big\|^2\\
\ge&
\mathbb{E}\big[\mathrm{Var}(A\mid \mu)\big]+\mathbb{E}\big[\mathrm{Var}(B\mid \mu)\big],
\end{aligned}
\end{equation*}
since the last term is nonnegative. Thus
\[
L_{\mathrm{FM}}(\theta)\;\ge\;\mathbb{E}\big[\mathrm{Var}\big(h_\theta(e_{\mathrm{act}})\mid \mu\big)\big]\;+\;\mathbb{E}\big[\mathrm{Var}(\Delta\mid \mu)\big].
\]

By Lemma~\ref{Lemma: Var_Y > mVar_X} with \(Y=h_\theta(X)\), \(X=e_{\mathrm{act}}\), we have
\[
\mathrm{Var}\big(h_\theta(e_{\mathrm{act}})\mid \mu\big)\ \ge\ m_h^2\,\mathrm{Var}(e_{\mathrm{act}}\mid \mu)\qquad\text{a.s.}
\]

Taking expectations and using Lemma~\ref{Lemma: mean under exchange} which gives \(\mathbb{E}[e_{\mathrm{act}}\mid \mu]=\mu\) and
\[
\mathbb{E}\big[\mathrm{Var}(e_{\mathrm{act}}\mid \mu)\big]
=\mathbb{E}\|e_{\mathrm{act}}-\mu\|^2
=\frac{1}{9}\,\mathbb{E}\!\Big[\|e_{\mathrm{act}}-c_{\mathrm{img}}\|_2^2
+\|e_{\mathrm{act}}-c_{\mathrm{lang}}\|_2^2
+\|c_{\mathrm{img}}-c_{\mathrm{lang}}\|_2^2\Big],
\]
we obtain
\[
\mathbb{E}\big[\mathrm{Var}\big(h_\theta(e_{\mathrm{act}})\mid \mu\big)\big]
\ \ge\
\frac{m_h^2}{9}\,\mathbb{E}\!\Big[\|e_{\mathrm{act}}-c_{\mathrm{img}}\|_2^2
+\|e_{\mathrm{act}}-c_{\mathrm{lang}}\|_2^2
+\|c_{\mathrm{img}}-c_{\mathrm{lang}}\|_2^2\Big].
\]

Finally, by definition of \(u^\star(\mu)=\mathbb{E}[\Delta\mid \mu]\),
\[
\mathbb{E}\big[\mathrm{Var}(\Delta\mid \mu)\big]
=\mathbb{E}\big\|\Delta-u^\star(\mu)\big\|^2.
\]

Combining the displays yields
\[
L_{\mathrm{FM}}(\theta)\;\ge\;\mathbb{E}\big\|\Delta-u^\star(\mu)\big\|^2+\frac{1}{9}\,\mathbb{E}\!\Big[\|e_{\mathrm{act}}-c_{\mathrm{img}}\|_2^2
+\|e_{\mathrm{act}}-c_{\mathrm{lang}}\|_2^2
+\|c_{\mathrm{img}}-c_{\mathrm{lang}}\|_2^2\Big].
\]
\end{proof}

\section{Additional Experiments on CALVIN}
\label{ap:add_exp}

\begin{table*}[!t]
 \centering
 \caption{Comparison with state-of-the-art methods on CALVIN under the ABC$\rightarrow$D protocol, evaluated by success rate and average success sequence length. Abbreviations indicate input modalities: S-RGB (static RGB), G-RGB (gripper RGB), S-RGBD (static RGB-D), G-RGBD (gripper RGB-D), P (proprioceptive arm state), and Cam (camera parameters).}
  \label{tab:calvin}
  \small
  \scalebox{0.9}{
  \begin{tabular}{c|c|ccccc|c}
\toprule    
    \multirow{2}{*}{\bf Method} & \multirow{2}{*}{\bf Input} & \multicolumn{6}{|c}{\bf No. Instructions in a Row (1000 chains)} \\ 
    \cline{3-8}
    & & 1 & 2 & 3 & 4 & 5 & Avg.Len. \\ \hline
    RoboFlamingo~\cite{li2023vision}  &S-RGB,G-RGB &	82.4\%	&61.9\%	& 46.6\%&	33.1\%	&23.5\%	& 2.47 \\
    GR-1~\cite{wu2023unleashing} & S-RGB,G-RGB,P &	85.4\%	&71.2\%	 & 59.6\%	&49.7\%	& 40.1\%&	3.06 \\
    3D Diffuser~\cite{ke20243d} &	S-RGBD,G-RGBD,P,Cam&	92.2\%	&78.7\%	& 63.9\%	& 51.2\%	& 41.2\%	& 3.27 \\
    GR-MG~\cite{li2025gr} &	S-RGBD,G-RGBD,P&	{\bf 96.8\%}	&{\bf 
 89.3\%}	&  {\bf 81.5\%} 	& {\bf 72.7}\%	& {\bf 64.4} \%	& {\bf 4.04} \\
 \hline
    SuSIE~\cite{black2023zero}&	 S-RGB &	87.0\%&	69.0\%	& 49.0\%	&38.0\%	 &26.0\%& 	2.69 \\
    GHIL-Glue~\cite{hatch2025ghil,black2023zero} & S-RGB & {\bf 95.2\%}	& {\bf 88.5\%}	& {\bf 73.2\%} & {\bf 62.5\%}& {\bf 49.8\%} & {\bf 3.69} \\
    Dita \cite{hou2025dita} & S-RGB &  {\bf 94.5\%}  & {\bf 82.5\%}  & {\bf 72.8\%} & {\bf 61.3\% } & {\bf 50.0\%} & {\bf 3.61} \\
    AGFT (Ours) & S-RGB &  {\bf 95.3\%}  & {\bf 85.3\%}  & {\bf 73.1\%} & {\bf 62.7\% } & {\bf 51.2\%} & {\bf 3.74} \\
\bottomrule
\end{tabular}}
\end{table*}

\textbf{Benchmarks.}
CALVIN \citep{mees2022calvin} is an open-source simulated benchmark for long-horizon, language-conditioned robotic manipulation. It provides four environments (Scenes A–D) and a standard cross-scene generalization protocol, denoted ABC$\rightarrow$D, in which policies are trained on Scenes A, B, and C and evaluated on the held-out Scene D. The benchmark comprises up to 1,000 task sequences, each formed by concatenating five elementary subtasks, thereby emphasizing compositional and temporally extended control. Performance is measured by the success sequence length, i.e., the number of consecutive subtasks completed within a sequence (up to five).
To evaluate the performance, we follow the ABC$\rightarrow$D protocol and restrict perception to static RGB observations.

\textbf{Performance.} Table~\ref{tab:calvin} summarizes the results on CALVIN. Some baselines exploit additional supervision or capacity: GR-MG leverages \emph{play} data during training, and GHIL-Glue~\cite{hatch2025ghil,black2023zero}—which extends SuSIE~\cite{black2023zero} with further finetuned generative modules~\cite{brooks2023instructpix2pix,xing2024dynamicrafter}—uses substantially larger model capacity, increasing compute and parameter counts.

For a controlled comparison, we follow the same training protocol as Dita and use the identical dataset. Under this matched setting, AGFT consistently outperforms Dita and remains competitive with strong prior baselines. Overall, the results indicate that AGFT captures the fine-grained visual cues required for long-horizon manipulation and generalizes robustly across scenes. We attribute these gains to explicit cross-modal alignment, which strengthens the coupling among vision, language, and action representations, leading to more stable action selection and improved downstream task adaptation.

\section{More Ablations}

\begin{figure}
    \centering
    \includegraphics[width=0.8\linewidth]{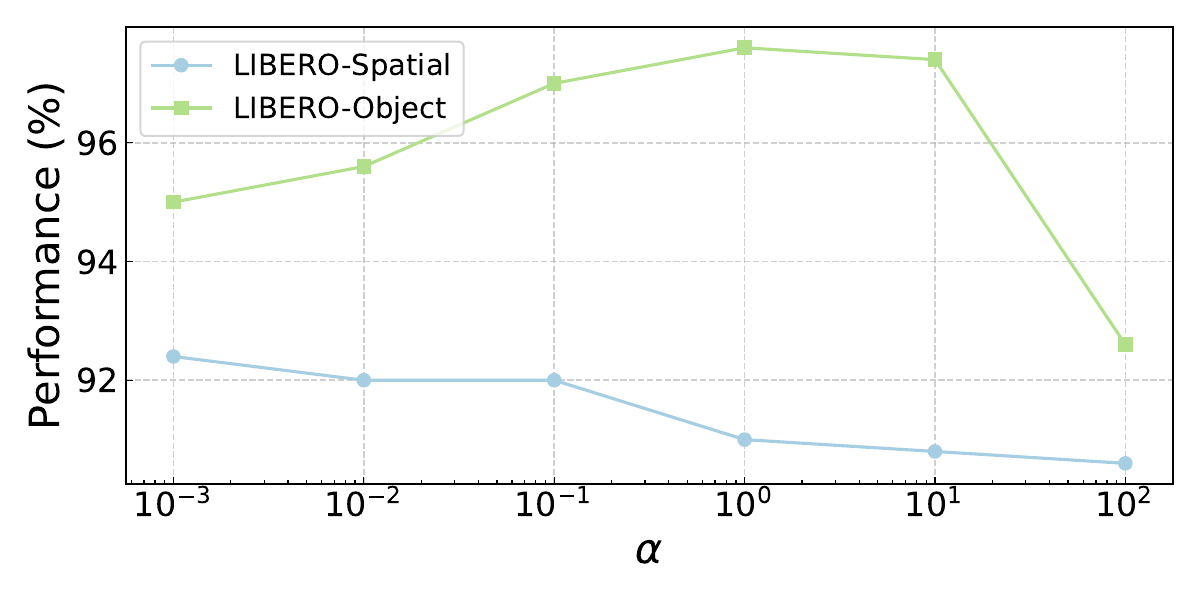}
    \vspace{-.5cm}
    \caption{Performance evaluation of different hyper-parameters $\alpha$ on the LIBERO-Spatial and LIBERO-Object tasks.}
    \label{fig:hyper}
    \vspace{-.4cm}
\end{figure}

\paragraph{Hyper-parameters.} We evaluate the effect of the weighting coefficient $\alpha$, which balances the contribution of the alignment-guided contrastive loss relative to the primary flow matching objective. Experiments are conducted on the LIBERO-Spatial and LIBERO-Object suites, as these tasks emphasize spatial reasoning and fine-grained object interactions—two aspects most sensitive to multimodal alignment strength. The results are summarized in Figure \ref{fig:hyper}.
Overall, AGFT exhibits strong robustness to variations in $\alpha$ across several orders of magnitude, confirming the stability of the joint training scheme. Performance remains consistently high within the range $\alpha \in [0.001, 1]$, where moderate weighting achieves an effective balance between flow consistency and semantic alignment. Notably, excessively large weights (e.g., $\alpha \geq 10$) begin to degrade performance, indicating that overemphasizing contrastive supervision may disrupt the optimization dynamics of the flow-matching field.
The best overall performance is observed at $\alpha = 0.001$ for LIBERO-Spatial and $\alpha = 1$ for LIBERO-Object, achieving success rates of 92.4\% and 97.6\%, respectively. These findings suggest that a balanced contribution of the alignment term yields the most favorable trade-off between semantic grounding and motion precision. 

%% file: checklist.tex
\section*{NeurIPS Paper Checklist}

\begin{enumerate}

\item {\bf Claims}
    \item[] Question: Do the main claims made in the abstract and introduction accurately reflect the paper's contributions and scope?
    \item[] Answer: \answerYes{} 
    \item[] Justification: 

\item {\bf Limitations}
    \item[] Question: Does the paper discuss the limitations of the work performed by the authors?
    \item[] Answer: \answerYes{} 
    \item[] Justification: Please refer to Appendix.

\item {\bf Theory assumptions and proofs}
    \item[] Question: For each theoretical result, does the paper provide the full set of assumptions and a complete (and correct) proof?
    \item[] Answer: \answerYes{} 
    \item[] Justification: Please refer to Section \ref{sec:theoretic} and \ref{appsec:theoretical}.

    \item {\bf Experimental result reproducibility}
    \item[] Question: Does the paper fully disclose all the information needed to reproduce the main experimental results of the paper to the extent that it affects the main claims and/or conclusions of the paper (regardless of whether the code and data are provided or not)?
    \item[] Answer: \answerYes{} 
    \item[] Justification: Please refer to Section \ref{sec:exp} and \ref{appsec:model}.

\item {\bf Open access to data and code}
    \item[] Question: Does the paper provide open access to the data and code, with sufficient instructions to faithfully reproduce the main experimental results, as described in supplemental material?
    \item[] Answer: \answerNo{} 
    \item[] Justification: We will release our code upon acceptance.

\item {\bf Experimental setting/details}
    \item[] Question: Does the paper specify all the training and test details (e.g., data splits, hyperparameters, how they were chosen, type of optimizer) necessary to understand the results?
    \item[] Answer: \answerYes{} 
    \item[] Justification: Please refer to Section \ref{sec:exp} and \ref{appsec:model}.

\item {\bf Experiment statistical significance}
    \item[] Question: Does the paper report error bars suitably and correctly defined or other appropriate information about the statistical significance of the experiments?
    \item[] Answer: \answerNo{} 
    \item[] Justification: While we report the average success rates over three independent evaluation runs with different random seeds to ensure reliability, we do not explicitly include error bars in the main tables to maintain clarity and comparability with prior baselines, many of which are not open-sourced and only provide mean performance.

\item {\bf Experiments compute resources}
    \item[] Question: For each experiment, does the paper provide sufficient information on the computer resources (type of compute workers, memory, time of execution) needed to reproduce the experiments?
    \item[] Answer: \answerYes{} 
    \item[] Justification: Please refer to Section \ref{sec:exp} and \ref{appsec:model}.
    
\item {\bf Code of ethics}
    \item[] Question: Does the research conducted in the paper conform, in every respect, with the NeurIPS Code of Ethics \url{https://neurips.cc/public/EthicsGuidelines}?
    \item[] Answer: \answerYes{} 
    \item[] Justification: 

\item {\bf Broader impacts}
    \item[] Question: Does the paper discuss both potential positive societal impacts and negative societal impacts of the work performed?
    \item[] Answer: \answerYes{} 
    \item[] Justification: Please refer to Section \ref{sec:con}.
    
\item {\bf Safeguards}
    \item[] Question: Does the paper describe safeguards that have been put in place for responsible release of data or models that have a high risk for misuse (e.g., pre-trained language models, image generators, or scraped datasets)?
    \item[] Answer: \answerNA{} 
    \item[] Justification: 

\item {\bf Licenses for existing assets}
    \item[] Question: Are the creators or original owners of assets (e.g., code, data, models), used in the paper, properly credited and are the license and terms of use explicitly mentioned and properly respected?
    \item[] Answer: \answerYes{} 
    \item[] Justification: 

\item {\bf New assets}
    \item[] Question: Are new assets introduced in the paper well documented and is the documentation provided alongside the assets?
    \item[] Answer: \answerNA{} 
    \item[] Justification: 

\item {\bf Crowdsourcing and research with human subjects}
    \item[] Question: For crowdsourcing experiments and research with human subjects, does the paper include the full text of instructions given to participants and screenshots, if applicable, as well as details about compensation (if any)? 
    \item[] Answer: \answerNA{} 
    \item[] Justification: 

\item {\bf Institutional review board (IRB) approvals or equivalent for research with human subjects}
    \item[] Question: Does the paper describe potential risks incurred by study participants, whether such risks were disclosed to the subjects, and whether Institutional Review Board (IRB) approvals (or an equivalent approval/review based on the requirements of your country or institution) were obtained?
    \item[] Answer: \answerNA{} 
    \item[] Justification: 

\item {\bf Declaration of LLM usage}
    \item[] Question: Does the paper describe the usage of LLMs if it is an important, original, or non-standard component of the core methods in this research? Note that if the LLM is used only for writing, editing, or formatting purposes and does \emph{not} impact the core methodology, scientific rigor, or originality of the research, declaration is not required.
    \item[] Answer: \answerNA{} 
    \item[] Justification: 

\end{enumerate}